\documentclass[accepted]{uai2026} 
                        
\usepackage[british]{babel}

\usepackage{amsmath, amssymb, amsthm, graphicx, bm}
\usepackage{hyperref}

\usepackage{caption}
\usepackage{subcaption}

\usepackage[ruled,vlined]{algorithm2e}

\usepackage{booktabs}
\usepackage{multirow}
\usepackage{array}
\usepackage{siunitx}
\DeclareMathOperator*{\argmin}{arg\,min}

\theoremstyle{plain}
\newtheorem{theorem}{Theorem}

\newtheorem{corollary}{Corollary}

\theoremstyle{definition}

\theoremstyle{remark}

\usepackage{natbib} 
\usepackage{mathtools} 
\usepackage{tikz} 

\title{Conformal Risk Sharing: Certified Cost Allocation with Participation Guarantees}

\author[ ]{Ieva Kazlauskaite}
\affil[ ]{%
    Department of Statistics\\
    London School of Economics and Political Science\\
    United Kingdom
}
  
\begin{document}
\maketitle

\begin{abstract}
Sharing the financial impact of rare adverse events across a group can soften extreme individual burdens, but any participant made worse off by the arrangement has reason to leave. A credible mechanism must therefore provide each agent with a trustworthy cap on their future obligation and should be deployed only if the aggregate harm across participants is bounded. We formalise this as the Certified Allocation Problem: from finite data and without distributional assumptions, find a redistribution rule, produce obligation caps for every participant, and verify that no participant is made materially worse off. We propose Conformal Risk Sharing, which solves this problem by pairing an interpretable sharing policy with split conformal calibration. The sharing intensity is tuned on training data, while held-out calibration data produces distribution-free per-agent guarantees (valid under exchangeability). Experiments on synthetic and real-world data, including precipitation and energy-cooperative data, confirm that the framework can substantially reduce extreme obligations for high-risk agents while controlling harm to others.
\end{abstract}

\section{Introduction}
\label{sec:intro}

A recurring problem in multi-agent systems is the redistribution of realised costs under uncertainty. When a group of agents jointly faces random shocks, an allocation policy determines how the realised burden is shared. Effective pooling can reduce each agent's exposure to extreme outcomes but it also creates winners and losers, and any agent made worse off has an incentive to defect. The challenge is therefore threefold: learn a policy that improves aggregate tail exposure from finite data, certify that the improvement holds with high confidence for each participant, and verify that the policy respects explicit participation constraints.

This problem arises naturally in several domains. In peer-to-peer (P2P) parametric insurance, trigger events (\emph{e.g.}, rainfall exceeding a threshold) induce a vector of payouts across members, and the pooling rule determines how the resulting obligations are shared. Heavy tails, spatial dependence, and climate nonstationarity make parametric tail models unreliable \citep{grossi2005catastrophe}, while participation is voluntary: members who perceive that pooling increases their high-confidence obligation cap will leave, and the resulting adverse selection can unravel the pool \citep{kocherlakota1996implications,ligon2005formation}. Fairness is equally critical: if pooling systematically increases the upper bound on obligations of low-risk members, those members exit first, unravelling the pool (see Appendix~\ref{sec:p2p_parametric_appendix}). In \emph{cooperative energy communities}, households in a local energy cooperative share electricity costs, and each household requires a certified cap on its obligation to plan expenditure.
In \emph{shared compute infrastructure}, multiple tenants redistribute cloud-resource costs after demand is realised, and each tenant needs a certified spending cap for capacity planning. In each case, the essential structure is the same: a random cost vector must be redistributed under conservation, every agent needs an individual tail guarantee, and the mechanism must remain acceptable relative to a baseline option.

We call this the \emph{Certified Allocation Problem} and formalise it in Sec.~\ref{sec:problem}. It sits at the intersection of cooperative cost sharing, distribution-free statistical inference, and mechanism design under uncertainty, none of which alone addresses all three requirements. Cooperative game theory and actuarial risk sharing \citep{denuit2022risk,charpentier2025linear} characterise fair and efficient allocations under known or assumed distributions, but do not produce finite-sample tail certificates from data. Conformal prediction and distribution-free risk control \citep{vovk2005algorithmic,angelopoulos2024conformal}
provide finite-sample guarantees, but address single-agent prediction or decision problems rather than multi-agent redistribution with participation constraints. Distributionally robust and chance-constrained optimisation \citep{campi2011sampling,delage2010distributionally} provide feasibility guarantees for a single decision under uncertainty. The Certified Allocation Problem requires $n$ simultaneous per-agent tail certificates, all depending on the same policy and coupled through a conservation constraint. This multi-agent structure does not reduce to $n$ independent chance constraints, since improving one agent's obligation cap necessarily affects others.

\paragraph{Conformal Risk Sharing.}
We propose Conformal Risk Sharing, a concrete solution framework for the Certified Allocation Problem. An interpretable linear sharing policy is tuned on training data, then certified once on held-out data via split conformal calibration. The result is a \emph{certified per-agent obligation cap} with finite-sample, distribution-free validity: each participant receives a high-confidence upper bound on future obligation, without relying on parametric tail models. The mechanism may only be deployed if an explicit audit confirms that the certified caps satisfy a participation constraint bounding harm to each agent.

Our main contributions are the following:
\begin{enumerate}[leftmargin=*,itemsep=2pt,topsep=2pt]
\item \textbf{Problem formulation.} We identify a gap at the intersection of cooperative cost sharing, conformal inference, and robust optimisation that no existing framework addresses. We formalise the Certified Allocation Problem (Sec.~\ref{sec:problem}): jointly select a redistribution policy, produce per-agent distribution-free obligation caps, and verify participation constraints from finite data. 
\item \textbf{Solution framework.} We propose Conformal Risk Sharing (Sec.~\ref{sec:method}), a train--select--certify pipeline that produces per-agent tail certificates (Theorem~\ref{thm:main}) and system-level guarantees (Corollary~\ref{cor:system}) via split conformal calibration, with participation constraints enforced directly in the certified quantities.
\item \textbf{Empirical validation.} On synthetic heavy-tailed data, gridded precipitation losses, and energy cooperative consumption (Sec.~\ref{sec:experiments}), the framework delivers certified tail relief while controlling harm and maintaining coverage at or near nominal levels.
\end{enumerate}

\section{Related Work}
\label{sec:related}

We review related work in cooperative cost sharing, distribution-free inference, and optimisation under uncertainty, and clarify how the Certified Allocation Problem differs from each.

\paragraph{Risk sharing, P2P insurance, and participation constraints.}
The actuarial literature studies risk-sharing rules under known or modelled distributions. \citet{denuit2022risk} provide a systematic treatment of allocation rules (conditional mean, quantile) and their axiomatic properties; \citet{charpentier2025linear} develop a comprehensive framework for linear risk sharing on networks, including the same identity-mixing parameterisation that we adopt; and \citet{feng2023peer} derive closed-form variance-minimising allocations. Throughout this literature, distributions are assumed known or estimated separately, and guarantees are expressed in terms of population quantities (means, variances, quantiles) rather than finite-sample certificates.
Participation and stability are classical concerns: voluntary pooling can unravel if participants lack individual rationality guarantees~\citep{kocherlakota1996implications,ligon2005formation}, and modern P2P designs introduce cashback or side-payment rules to maintain incentives~\citep{clemente2023optimal}. Cooperative game theory formalises fairness in cost sharing through solution concepts such as the core, Shapley value, and nucleolus, which characterise stable and equitable allocations under known cost structures \citep[see][for an overview]{moulin2002chapter}. Our participation constraints serve an analogous role but are expressed in certified tail quantities rather than population-level cost shares.

\paragraph{Conformal prediction and risk control.}
Conformal prediction provides finite-sample, distribution-free coverage
guarantees under exchangeability
\citep{vovk2005algorithmic,shafer2008tutorial,angelopoulos2023conformal}. Beyond marginal coverage, Conformal Risk Control \citep{angelopoulos2024conformal} and related methods \citep{bates2021distribution} extend these guarantees to user-chosen risk functionals. Our use differs from standard supervised prediction: the calibrated quantity is a post-decision random variable (per-agent obligation after applying an allocation rule), not a prediction error. A growing literature treats conformal outputs as inputs to downstream decisions \citep{vovk2018conformal,renkema2024conformal}, and several works connect conformal sets to robust optimisation \citep{johnstone2021conformal,lekeufack2024conformal,patel2026conformal}. These pipelines address single-agent or single-decision-vector optimisation problems. 
In the multi-agent setting, \citet{kuipers2024conformal} derive conformal joint prediction regions for agent trajectories under policy-induced distribution shift using reweighting ideas, which is related in spirit but targets off-policy trajectory forecasting rather than certified cost redistribution.
End-to-end conformal risk training \citep{yeh2025conformal,yeh2024end} differentiates through conformal objectives to shape decisions, but targets single-agent control with expressive uncertainty sets rather than multi-agent redistribution with participation constraints.
Our framework embeds conformal calibration inside a decision pipeline, with a train--select--certify separation \citep{vovk2018conformal,sarkar2023post,hegazy2025valid}: all policy selection uses training/validation data, and calibration is used only once for the final certificate.

\paragraph{DRO, chance constraints, and scenario optimisation.}
Chance-constrained optimisation, scenario methods, and distributionally robust optimisation (DRO) study decision-making under uncertainty by selecting a policy that satisfies probabilistic or worst-case constraints under sampled scenarios or an ambiguity set \citep{campi2011sampling,delage2010distributionally,rahimian2019distributionally}. Our setting differs in the object being certified: we require simultaneous per-agent tail certificates for all agents under a shared allocation policy, with agent outcomes coupled by conservation. In addition, participation is enforced by comparing each agent's certified bound under pooling against their baseline bound, a structure with no standard counterpart in DRO or scenario formulations.

\paragraph{Conformal inference under dependence and nonstationarity.}
In our settings, the data are temporally and/or spatially dependent, so we calibrate on coarse
blocks (\emph{e.g.}, years) and assume block exchangeability. There is growing work extending conformal ideas beyond i.i.d.\ exchangeability~\citep{chernozhukov2018exact,oliveira2024split}, and methods that reweight or adapt calibration sets to account for temporal drift have been proposed for financial time series~\citep{fantazzini2024adaptive,schmitt2026taming}. We treat certificates as valid for blocks exchangeable with the calibration regime and recommend periodic re-certification under drift.

\section{The Certified Allocation Problem}
\label{sec:problem}

We formalise the Certified Allocation Problem class introduced in
Sec.~\ref{sec:intro}. The setup requires guarantees derived directly from finite data without parametric assumptions, for a multi-agent redistribution policy subject to conservation and participation constraints.

\subsection{Setup}

Consider $n$ agents who collectively face a random nonnegative cost vector
$\tilde X \in \mathbb{R}^n_+$ drawn from an unknown distribution~$P$. We observe $B$ exchangeable realisations (``blocks'') of this vector, $\tilde x_1, \dots, \tilde x_B$, where
each block aggregates costs over a natural period (\emph{e.g.}, a year of parametric
insurance payouts, a billing cycle of shared compute costs, or a settlement
period in a cooperative energy community).

\paragraph{Allocation policies.}
An \emph{allocation policy} $A$ maps a realised cost vector to a vector of
\emph{obligations}:
\begin{equation}
  x_b(A) = \tilde x_b\, A, \label{eq:alloc}
\end{equation}
where $A \in \mathbb{R}^{n \times n}$ acts on the right so that agent~$i$'s
obligation is $x_{b,i}(A) = \sum_j \tilde x_{b,j} A_{ji}$. We restrict
attention to the feasible set of \emph{row-stochastic} matrices,
\begin{equation}
  \mathcal{A} = \Big\{ A \in \mathbb{R}^{n \times n} :\;
    A_{ji} \ge 0 \;\forall\, j,i, \quad
    \textstyle\sum_{i=1}^n A_{ji} = 1 \;\forall\, j \Big\},
  \label{eq:feasible}
\end{equation}
which enforces nonnegativity of obligations and conservation of total cost
within each block: $\sum_i x_{b,i}(A) = \sum_j \tilde x_{b,j}$.%
    \footnote{Row-stochasticity ($\sum_i A_{ji} = 1 \; \forall j$) ensures that each agent's loss is fully distributed but it does not constrain how much total exposure each agent receives. The stronger requirement of a \emph{doubly stochastic} $A$ ($\sum_i A_{ji} = 1 \; \forall j$ and $\sum_j A_{ji} = 1 \; \forall i$) additionally preserves each agent's expected cost~\citep{abdikerimova2022peer}. Our framework accommodates either constraint; we use row-stochasticity as the default to allow heterogeneous pooling structures.}
The baseline is the identity $A_0 = I$, under which each
agent bears their own realised cost.

\paragraph{Per-agent risk quantities.}
Each agent cares about the severity of extreme obligations they may face under a given policy. For a given policy $A$, let $X_i(A)$ denote agent $i$'s random obligation, and define $\rho_i(A)$ as the $(1{-}\delta)$-quantile of $X_i(A)$, \emph{i.e.} the smallest threshold exceeded with probability at most $\delta$ (sometimes called Value-at-Risk in the actuarial literature). The formulation extends to other tail risk measures (\emph{e.g.}, Conditional Value-at-Risk, expectiles), though certifying other choices requires different calibration procedures. The quantity $\rho_i(A)$ is a population-level summary estimated from data, and the quality of the estimate depends on $A$ since the policy transforms the underlying loss distribution.

\subsection{Requirements}

Given a finite sample of $B$ blocks, the mechanism seeks a policy
$A^\star \in \mathcal{A}$ together with per-agent caps
$\hat c_i(A^\star)$ satisfying the following requirements.

\paragraph{(R1) Per-agent tail validity.} Each \emph{certificate} is a finite-sample high-probability upper bound on the agent's post-allocation obligation:
  \begin{equation}
    P\big(X_i(A^\star) \le \hat c_i(A^\star)\big) \ge 1 - \delta,
    \qquad i = 1, \dots, n,
    \label{eq:coverage}
  \end{equation}
  without parametric assumptions on $P$. This provides every participant a
  legible guarantee: \emph{``with $1{-}\delta$ confidence, your obligation will not
  exceed $\hat c_i$.''} In P2P insurance this is a contribution cap, in energy communities it is a billing guarantee.

\paragraph{(R2) Aggregate efficiency.} The policy reduces an aggregate welfare
  objective relative to the baseline:
  \begin{equation}
    \Phi(A^\star) \;<\; \Phi(A_0), \qquad
    \Phi(A) := \textstyle\sum_{i=1}^n w_i\, \hat c_i(A),
    \label{eq:efficiency}
  \end{equation}
  where $w_i > 0$ are weights (normalised to $\sum_i w_i = 1$) that allow the mechanism to prioritise certain agents, \emph{e.g.}, uniformly ($w_i = 1/n$) or proportional to baseline exposure.

\paragraph{(R3) Participation (bounded harm).} The total certified harm
  imposed on agents relative to the baseline is bounded. We require a \emph{budgeted harm constraint}:
  \begin{equation}
    \mathrm{Harm}(A^\star) = \textstyle\sum_{i=1}^n w_i
    \big(\hat c_i(A^\star) - \hat c_i(A_0) - \eta\big)_+
    \;\le\; H,
    \label{eq:harm}
  \end{equation}
  with per-agent tolerance $\eta \ge 0$ and total budget
  $H = \varepsilon\, \Phi(A_0)$, a fraction $\varepsilon \in [0,1)$ of the baseline aggregate certified cost (the special case $\varepsilon = 0$ requires that no agent's certified cap increases, which is typically infeasible with finite calibration samples). These are governance parameters, not statistical hyperparameters: $\eta$ is a materiality threshold below which cap increases are ignored and $\varepsilon$ controls how much redistribution the community tolerates. 

\paragraph{(R4) Conservation.} Total obligations equal total realised costs in every block, ensuring the mechanism neither creates nor destroys value. This is enforced by construction through the row-stochasticity constraint~\eqref{eq:feasible}.

\subsection{Challenges}
\label{sec:why-hard}

Three features distinguish the Certified Allocation Problem from
related formulations.

\emph{Selection--certification coupling.} The policy $A^\star$ determines the distribution of obligations, so the data used to select $A^\star$ also informs the certificates. In distribution-free approaches, adaptively choosing a decision based on calibration data can invalidate coverage guarantees \citep{vovk2018conformal,hegazy2025valid}, requiring careful separation of learning and certification. Parametric and Bayesian approaches avoid this issue but introduce dependence on model assumptions.

\emph{Multi-agent coupling through conservation.} The conservation constraint $\sum_i x_{b,i}(A) = \sum_j \tilde x_{b,j}$ means that in any given block, reducing one agent's obligation necessarily increases another's. When losses are largely independent, diversification can reduce \emph{tail} exposure for all agents simultaneously; but under dependence, the gains are limited and improving one agent's certified cap often comes at the cost of worsening another's. This coupling between agents distinguishes the problem from single-agent robust optimisation, where the decision-maker can reduce their own uncertainty without affecting others.

\emph{Heavy tails and small samples.} In the motivating applications, block losses are heavy-tailed and the number of observed blocks $B$ is small (tens to low hundreds). Parametric tail models are difficult to validate in this regime making distribution-free guarantees attractive, but at the cost of conservatism as distribution-free certificates from small $B$ can be coarse. Moreover, the quality of these certificates depends on the chosen policy $A$ (since $A$ transforms the loss distribution), creating an interaction between policy selection and certificate tightness.

\subsection{Positioning Relative to Existing Frameworks}
\label{sec:positioning}

\begin{table}[t]
\centering
\caption{Requirements addressed by existing frameworks.
  $\checkmark$ = addressed; $\times$ = not; ${\sim}$ = partially.}
\label{tab:gap}
\footnotesize
\setlength{\tabcolsep}{4pt}
\begin{tabular}{lccc}
\toprule
& \textbf{R1} Tail & \textbf{R2} Effic. & \textbf{R3} Partic. \\
\midrule
Coop.\ games / actuarial\textsuperscript{a}
  & $\times$ & $\checkmark$ & $\checkmark$ \\
Conformal decision-making\textsuperscript{b}
  & $\checkmark$ & $\checkmark$ & $\times$ \\
DRO / scenario approach\textsuperscript{c}
  & $\sim$ & $\checkmark$ & $\times$ \\
\textbf{This work}
  & $\checkmark$ & $\checkmark$ & $\checkmark$ \\
\bottomrule
\end{tabular}
\par\vspace{4pt}
{\scriptsize
\textsuperscript{a}\citet{denuit2022risk,charpentier2025linear,feng2023peer,kocherlakota1996implications}.\;
\textsuperscript{b}\citet{vovk2005algorithmic,johnstone2021conformal,lekeufack2024conformal}.\;
\textsuperscript{c}\citet{campi2011sampling,delage2010distributionally}.}
\end{table}

Table~\ref{tab:gap} summarises which requirements are addressed by existing
frameworks. Conservation (R4) is
enforced structurally and is therefore omitted from the comparison. Among the remaining requirements, no single framework covers all three; the Certified Allocation Problem requires their integration.

\section{Conformal Risk Sharing}
\label{sec:method}

We now present a concrete solution to the Certified Allocation Problem based on split conformal prediction. Our approach combines an interpretable one-parameter linear policy class with distribution-free order-statistic certificates. The base pooling rule $\bar A$ encodes the structure of sharing and can itself be learned from data (Sec.~\ref{sec:protocol}); the scalar $\alpha$ controls the intensity of mutualisation and is selected by grid search subject to certified participation constraints. Alternative solution approaches (Bayesian, parametric) are discussed in Sec.~\ref{sec:discussion}.

\subsection{Policy Class}
\label{sec:policy-class}

We parameterise the allocation as a convex combination of the identity (no
pooling) and a base pooling rule $\bar A \in \mathcal{A}$, controlled by a
single scalar mutualisation level $\alpha \in [0,1]$:
\begin{equation}
  A(\alpha) = (1 - \alpha)\,I + \alpha\,\bar A.
  \label{eq:alpha-policy}
\end{equation}
If both $I$ and $\bar A$ are row-stochastic, so is $A(\alpha)$ for all
$\alpha$, and feasibility~\eqref{eq:feasible} is satisfied by construction.
The policy has a contract-like interpretation: each agent retains a $(1{-}\alpha)$ fraction of their own cost and routes an $\alpha$ fraction through the pooling mechanism. This interpretability is deliberate: in insurance and cost-sharing agreements, policies must be auditable and comprehensible.

The base rule $\bar A$ encodes the \emph{structure} of sharing (who pools with whom) and may be either fixed by design or learned from data. Fixed choices include a uniform pool ($\bar A_{ji} = 1/n$), a locality kernel on a spatial grid, or a sparse neighbourhood rule. Alternatively, $\bar A$ can be estimated from training data; for example, the variance-optimal doubly-stochastic baseline (VO-DS) selects $\bar A$ by minimising a quadratic variance proxy under fairness constraints (Appendix~\ref{sec:VO-DS}). When $\bar A$ is data-driven it is fitted in Stage~1 of Algorithm~\ref{alg:tsc} before $\alpha$ selection. The scalar $\alpha$ then controls the \emph{intensity} of sharing: this separation allows domain experts or data-driven methods to design $\bar A$ while the certification layer tunes $\alpha$ subject to participation constraints. VO-DS targets variance rather than tail risk; within our framework it inherits a distribution-free certification layer.

\subsection{Conformal Certificates}
\label{sec:certificates}

We partition the $B$ observed blocks into training $\mathcal{T}$, validation
$\mathcal{V}$, and calibration $\mathcal{C}$ (with $|\mathcal{C}| = m$). When nonstationarity is a concern, blocks are partitioned in temporal order so that $\mathcal{C}$ is as close as possible to the deployment period (Appendix~\ref{sec:nonstationarity}).
For any \emph{fixed} policy $A$, we construct the per-agent certificates
$\hat c_i(A)$ required by \textbf{R1} as order statistics of the calibration
obligations:
\begin{equation}
\begin{aligned}
      \hat c_i(A) &= k\text{-th order statistic of }
  \{x_{b,i}(A)\}_{b \in \mathcal{C}},\\
  \quad k &= \lceil(m+1)(1-\delta)\rceil.
  \label{eq:conf-caps}
\end{aligned}
\end{equation}
Under block exchangeability between $\mathcal{C}$ and a fresh block,
$P(X_i(A) \le \hat c_i(A)) \ge 1 - \delta$ for each agent $i$
(Theorem~\ref{thm:main}). 
The selection procedure (Sec.~\ref{sec:protocol}) minimises the aggregate certified cost $\Phi(A) = \sum_i w_i\, \hat c_i(A)$ to target efficiency (\textbf{R2}), while the participation constraints \textbf{R3} are enforced by comparing $\hat c(A^\star)$ to the baseline $\hat c(A_0)$ via~\eqref{eq:harm}. If no feasible policy improves on the baseline, the method transparently reverts to $A_0 = I$.

\subsection{Train--Select--Certify Protocol}
\label{sec:protocol}

The non-differentiability of the conformal map $A \mapsto c(A)$ and the need to preserve calibration validity motivate a three-stage protocol
(Algorithm~\ref{alg:tsc}).

\paragraph{Stage 1: Policy-class fitting ($\mathcal{T}$).}
If the base rule $\bar A$ is data-driven (\emph{e.g.}, the variance-optimal doubly-stochastic baseline of Appendix~\ref{sec:VO-DS}), it is estimated on $\mathcal{T}$. When $\bar A$ is fixed by design (\emph{e.g.}, uniform pooling or a predetermined locality kernel), this stage is skipped and $\mathcal{T}$
may be merged into $\mathcal{V}$.

\paragraph{Stage 2: Mutualisation selection ($\mathcal{V}$).}
For each $\alpha$ on a grid over $[0,1]$, we estimate the per-agent risk quantity $\rho_i(A(\alpha))$ by its empirical counterpart on $\mathcal{V}$, hence $\hat\rho_i^{\mathcal{V}}(\alpha)$ is the empirical $(1{-}\delta)$-quantile of agent $i$'s post-sharing obligations $\{x_{b,i}(\alpha)\}_{b \in \mathcal{V}}$. We then compute proxy harm:
\begin{equation}
  \mathrm{Harm}_{\mathrm{proxy}}(\alpha)
  = \sum_{i=1}^n w_i \big(\hat\rho_i^{\mathcal{V}}(\alpha)
    - \hat\rho_i^{\mathcal{V}}(0) - \eta\big)_+.
  \label{eq:proxy-harm}
\end{equation}
Here $\hat\rho_i^{\mathcal{V}}(0)$ denotes the baseline empirical risk evaluated at $\alpha=0$ (the identity policy $A_0=I$, under which each agent bears their own cost). We retain the $\alpha^\star$ minimising $\sum_i w_i\, \hat\rho_i^{\mathcal{V}}(\alpha)$ subject to $\mathrm{Harm}_{\mathrm{proxy}}(\alpha) \le H_{\mathrm{proxy}}$, where $H_{\mathrm{proxy}} = \varepsilon\,\Phi_{\mathcal{V}}(A_0)$ is the proxy analogue of the harm budget~\eqref{eq:harm} with $\Phi(\cdot)$ as in~\eqref{eq:efficiency}. No certification claims are made at this stage.

\paragraph{Stage 3: One-shot certification ($\mathcal{C}$).}
We fix $A^\star = A(\alpha^\star)$, compute conformal caps $c(A^\star)$ and $c(A_0)$ on the untouched calibration set via~\eqref{eq:conf-caps}, and evaluate the certified harm constraint~\eqref{eq:harm}. If the constraint holds, $A^\star$ is deployed; otherwise the operator reverts to $A_0 = I$. This is an operational audit, not a statistical selection step: the certificate remains valid regardless of the deployment decision.

\begin{algorithm}[t]
\caption{Conformal Risk Sharing: Train--Select--Certify}
\label{alg:tsc}
\KwIn{Blocks $\{\tilde x_b\}_{b=1}^B$; base rule $\bar A$
  (or specification to learn it); miscoverage $\delta$;
  weights $w$; tolerance $\eta$; budget fraction $\varepsilon$.}
\KwOut{Deployed policy $A_{\mathrm{op}}$ and certificates
  $c(A_{\mathrm{op}})$.}
\BlankLine
Partition blocks into $\mathcal{T}$, $\mathcal{V}$, $\mathcal{C}$\;
\BlankLine
\tcp{\color{gray} \small Stage 1: Fit policy class on $\mathcal{T}$ (skip if $\bar A$ is fixed)}
\BlankLine
\tcp{\color{gray} \small Stage 2: Select $\alpha$ on $\mathcal{V}$}
\For{$\alpha$ in grid over $[0,1]$}{
  Compute proxy risk $\hat\rho_i^{\mathcal{V}}(\alpha)$ and
    proxy harm~\eqref{eq:proxy-harm}\;
}
$\alpha^\star \leftarrow$ best feasible $\alpha$
  (lowest $\sum_i w_i\,\hat\rho_i^{\mathcal{V}}(\alpha)$,
   proxy harm $\le$ budget)\;
$A^\star \leftarrow (1-\alpha^\star)I + \alpha^\star \bar A$\;
\BlankLine
\tcp{\color{gray} \small Stage 3: Certify on $\mathcal{C}$}
$\hat c_0 \leftarrow \mathrm{ConfCaps}(\mathcal{C},\, I,\, \delta)$\;
$\hat c^\star \leftarrow \mathrm{ConfCaps}(\mathcal{C},\, A^\star,\, \delta)$\;
$\mathrm{Harm} \leftarrow \sum_i w_i
  (\hat c_i^\star - \hat c_{0,i} - \eta)_+$\;
\BlankLine
\tcp{\color{gray} \small Operational audit}
\eIf{$\mathrm{Harm} \le \varepsilon \langle w, \hat c_0 \rangle$}{
  \Return $(A^\star,\, \hat c^\star)$
    \tcp*{\color{gray} \small Deploy candidate}
}{
  \Return $(I,\, \hat c_0)$
    \tcp*{\color{gray} \small Revert to baseline}
}
\end{algorithm}

\paragraph{What is (and is not) guaranteed.}
The certificate~\eqref{eq:conf-caps} provides a \emph{marginal}, per-agent guarantee: each $\hat c_i(A^\star)$ controls agent $i$'s obligation with probability at least $1{-}\delta$ over a fresh exchangeable block. It does \emph{not} imply the joint statement $P(\forall i: X_i(A^\star) \le \hat c_i(A^\star)) \ge 1{-}\delta$; for system-level solvency we provide a separate certificate (Corollary~\ref{cor:system}). The calibration set $\mathcal{C}$ is used exactly once for a pre-specified $A^\star$; the subsequent deployment decision does not invalidate the certificate.

\subsection{Rare-Event Degeneracy}
\label{sec:degeneracy}

In zero-inflated data (where many peers have no triggers in the calibration window), the baseline cap $\hat c_{0,i}$ can be zero, causing the harm budget $H = \varepsilon \langle w, \hat c_0 \rangle$ to vanish and the acceptability gate to reject any pooling. To prevent this, we impose a minimum capital floor $c_{\min} > 0$, replacing both candidate and baseline by $\max\{\hat c_i(A),\, c_{\min}\}$ componentwise. Since the floor only increases caps, coverage validity is preserved. In all experiments we set $c_{\min}$ to a small fraction (1\%) of the median nonzero loss, treating it as an operational minimum capital requirement: no agent should plan with zero reserves against tail events.

\subsection{Guarantees}
\label{sec:guarantees}

\begin{theorem}[Per-agent tail certificate]
\label{thm:main}
Fix $i\in\{1,\dots,n\}$ and $\delta\in(0,1)$. Let $\mathcal{T},\mathcal{V}, \mathcal{C}$ partition the observed blocks with $|\mathcal{C}|=m$, and let $A^\star$ be any allocation constructed from the training and validation blocks $\{\tilde X_b\}_{b\in\mathcal{T}\cup\mathcal{V}}$ (and, optionally, algorithmic randomness independent of the data), using no information from $\mathcal{C}$. Assume that the calibration blocks $\{\tilde X_b\}_{b\in\mathcal{C}}$ together with a fresh block $\tilde X_{\mathrm{new}}$ are exchangeable, conditional on $A^\star$.\footnote{This holds if $(\tilde X_1,\dots,\tilde X_B,\tilde X_{\mathrm{new}})$ are exchangeable at the block level and $A^\star$ uses no information from $\mathcal{C}$.} Define $\hat c_i(A^\star)$ as the $k$-th order statistic of $\{X_{b,i}(A^\star)\}_{b\in\mathcal{C}}$ with $k=\lceil(m+1)(1-\delta)\rceil$. Then
\begin{equation}
  P\big(X_{\mathrm{new},i}(A^\star) \le \hat c_i(A^\star)\big)
  \;\ge\; 1 - \delta.
  \label{eq:main-guarantee}
\end{equation}
\end{theorem}

\begin{proof}
Conditioning on $A^\star$, the $m+1$ values
$\{X_{b,i}(A^\star)\}_{b\in\mathcal{C}} \cup
\{X_{\mathrm{new},i}(A^\star)\}$ are exchangeable by assumption and the rank of $X_{\mathrm{new},i}(A^\star)$ among all $m+1$ values is uniform on $\{1,\dots,m+1\}$.\footnote{When ties occur, the rank is defined by uniform random tie-breaking; without tie-breaking the bound still holds since ties increase $\hat c_i(A^\star)$.} Since $\hat c_i(A^\star)$ is the $k$-th smallest of the $m$ calibration values, at most $m - k$ of the $m+1$ exchangeable values can exceed it. Hence $X_{\mathrm{new},i}(A^\star)$ exceeds $\hat c_i(A^\star)$ with probability at most $(m-k+1)/(m+1) \le \delta$, giving $P(X_{\mathrm{new},i}(A^\star) \le \hat c_i(A^\star) \mid A^\star) \ge 1-\delta$ for any fixed $A^\star$, implying~\eqref{eq:main-guarantee}.
\end{proof}

Theorem~\ref{thm:main} provides \textbf{R1} of the Certified Allocation Problem. Two remarks clarify the scope: \emph{(i) Marginal, not joint.} The guarantee~\eqref{eq:main-guarantee} is per-agent and does not imply the simultaneous statement $P(\forall i: X_i(A^\star) \le \hat c_i(A^\star)) \ge 1-\delta$. For system-level solvency, Corollary~\ref{cor:system} provides a separate certificate.
\emph{(ii) Operational audit.} The subsequent decision to deploy $A^\star$ or revert to $A_0$ based on the harm budget does not invalidate the certificate, since $A^\star$ was fixed before observing $\mathcal{C}$. The deploy/revert step is a governance decision informed by a valid statistical audit.

\begin{corollary}[System-level certificate]
\label{cor:system}
Under the assumptions of Theorem~\ref{thm:main}, let
$g:\mathbb{R}^n_+\to\mathbb{R}$ be any scalar functional (e.g.,
$g(x)=\sum_i x_i$ for aggregate cost or $g(x)=\max_i x_i$ for the largest realised obligation across agents) fixed prior to observing $\mathcal{C}$. Let $\Gamma(A^\star)$ be the $k$-th order statistic of $\{g(X_b(A^\star))\}_{b\in\mathcal{C}}$ with $k=\lceil(m+1)(1-\delta)\rceil$. Then
\begin{equation}
  P\big(g(X_{\mathrm{new}}(A^\star)) \le \Gamma(A^\star)\big)
  \;\ge\; 1 - \delta.
  \label{eq:system-guarantee}
\end{equation}
\end{corollary}

The proof is identical to Theorem~\ref{thm:main}, applied to the scalar scores
$S_b = g(X_b(A^\star))$.

\section{Experiments}
\label{sec:experiments}

We evaluate Conformal Risk Sharing on synthetic heavy-tailed data (Sec.~\ref{sec:exp-synth}), gridded precipitation losses (Sec.~\ref{sec:exp-eobs}), and electricity consumption of an energy cooperative (Sec.~\ref{sec:exp-energy})\footnote{Code available at \url{https://github.com/IevaKazlauskaite/conformal-risk-sharing}}.  In each case we ask whether: (i)~the conformal certificates maintain valid coverage (\textbf{R1})? (ii)~the learned policy reduces aggregate certified tail exposure (\textbf{R2})? (iii)~the certified harm is controlled within the participation budget (\textbf{R3})?

\subsection{Synthetic Data}
\label{sec:exp-synth}

We generate $B$ i.i.d.\ blocks (years) with $n$ peers on a grid. Each block is an event year with probability $p_{\mathrm{event}}$; on event years, peer losses are products of a heavy-tailed (Pareto) year severity, peer-specific lognormal exposures, and spatially correlated hit indicators drawn from a logistic model. The resulting loss vectors are nonnegative, heavy-tailed, and strongly zero-inflated (Appendix~\ref{app:synth-dgp}). We compare three base rules: \emph{global-uniform} pooling ($\bar A = \mathbf{1}\mathbf{1}^\top/n$), \emph{local} pooling (neighbourhood averaging on the grid), and a data-driven \emph{variance-optimal doubly-stochastic} baseline (VO-DS; Appendix~\ref{sec:VO-DS}). All experiments use $\delta = 0.10$
(nominal 90\%), $\eta = 0$, and $\varepsilon = 0.20$.

We report: (a)~the empirical per-agent marginal coverage aggregated across all test blocks and splits (mean and 5th percentile across peers); (b)~\emph{AggCapRatio} $= \langle w, c_{\mathrm{op}}\rangle / \langle w, c_0 \rangle$, measuring aggregate certified-cap reduction (lower is better); (c)~\emph{Top10 cap}, the cap ratio restricted to the top decile of agents by baseline cap (targeting relief for the highest-risk agents); and (d)~\emph{PASS rate}, the fraction of splits where the candidate survives the certified acceptability audit. All coverage and cap metrics are reported for the \emph{operational} policy $A_{\mathrm{op}}$, which equals the candidate $A^\star$ on PASS splits and reverts to identity on FAIL splits. Metrics are defined in Appendix~\ref{app:synth-metrics}.

Table~\ref{tab:synth_random_oper_identity} reports results under random splits. Empirical per-agent coverage is near nominal across pooling families and identity (mean $\approx 0.91$; 5th percentiles at or above $0.90$), consistent with the intended conformal validity guarantee (\textbf{R1}).
Global pooling achieves the largest certified relief for high-risk agents (Top10 ratio $0.907$, \emph{i.e.}, ${\approx}9\%$ tail capital reduction), while local pooling delivers more modest gains (Top10 $0.939$) (\textbf{R2}). The PASS rate of $0.77$ for both families indicates that the certification gate rejects roughly one quarter of candidate policies, confirming that the participation constraint is active (\textbf{R3}). The relatively large standard deviations on the cap ratios reflect split-level variability: on FAIL splits the method reverts to identity (ratio $= 1$), while
on PASS splits the candidate delivers meaningful relief.

Table~\ref{tab:synth_time_oper_identity} repeats the analysis under
time-ordered splits. Coverage degrades slightly (p05 drops to $0.884$), but identity degrades identically, confirming a nonstationarity effect rather than a method failure. The PASS rate increases to $1.00$ because the method selects conservative mutualisation levels under temporal ordering, easily fitting within
the harm budget. Certified relief is correspondingly modest. 

\begin{table}[t]
\centering
\setlength{\tabcolsep}{3pt}
\scriptsize
\begin{tabular}{lccccc}
\toprule
 & PASS & $\alpha_{\mathrm{op}}$ & Cov (p05) & AggCap & Top10 \\
\midrule
Global & 0.77 & 0.14\tiny$\pm$0.29 & 0.910 (0.900) & 0.970\tiny$\pm$0.070 & 0.907\tiny$\pm$0.198 \\
Local & 0.77 & 0.10\tiny$\pm$0.24 & 0.909 (0.897) & 0.984\tiny$\pm$0.044 & 0.939\tiny$\pm$0.148 \\
VO-DS & 0.77 & 0.14\tiny$\pm$0.29 & 0.910 (0.900) & 0.970\tiny$\pm$0.070 & 0.907\tiny$\pm$0.198 \\
Identity & -- & 0 & 0.910 (0.899) & -- & -- \\
\bottomrule
\end{tabular}
\caption{Synthetic, random splits ($\delta\!=\!0.1$, $\eta\!=\!0$,
$\varepsilon\!=\!0.2$, $n_C\!=\!100$, 100 splits). Cov: mean per-agent marginal coverage ($5^{\text{th}}$ percentile). AggCap/Top10: certified-cap ratios vs.\ identity (lower = better).}
\label{tab:synth_random_oper_identity}
\end{table}

\begin{table}[t]
\centering
\setlength{\tabcolsep}{3pt}
\scriptsize
\begin{tabular}{lccccc}
\toprule
 & PASS & $\alpha_{\mathrm{op}}$ & Cov (p05) & AggCap & Top10 \\
\midrule
Global & 1.00 & 0.08\tiny$\pm$0.26 & 0.910 (0.884) & 0.978\tiny$\pm$0.073 & 0.946\tiny$\pm$0.179 \\
Local & 1.00 & 0.05\tiny$\pm$0.17 & 0.910 (0.883) & 0.988\tiny$\pm$0.039 & 0.968\tiny$\pm$0.107 \\
VO-DS & 1.00 & 0.08\tiny$\pm$0.26 & 0.910 (0.884) & 0.978\tiny$\pm$0.073 & 0.946\tiny$\pm$0.179 \\
Identity & -- & 0 & 0.911 (0.884) & -- & -- \\
\bottomrule
\end{tabular}
\caption{Synthetic, time-ordered splits ($\delta\!=\!0.1$, $\eta\!=\!0$,
$\varepsilon\!=\!0.2$, $n_C\!=\!100$, 12 splits).}
\label{tab:synth_time_oper_identity}
\end{table}

\subsection{E-OBS Precipitation}
\label{sec:exp-eobs}

We use E-OBS daily gridded rainfall over a Central European region (lat
$40$--$50^\circ$, lon $5$--$12^\circ$), with $n = 1120$ grid cells observed over $B = 75$ annual blocks (1950 -2024)~\citep{EOBS,cornes2018ensemble}. To
mirror a parametric insurance design, we define a binary cold-season trigger: a unit payout is recorded whenever the Oct--Mar rainfall total exceeds a threshold $u = 40$ (chosen to yield a median per-cell annual trigger rate of $\approx 11\%$, with substantial cross-cell heterogeneity). The resulting block losses are heavy-tailed and highly zero-inflated (zero fraction $\approx 0.78$), with substantial spatial dependence. We use $\delta = 0.10$, $\eta = 0$, $\varepsilon = 0.20$ and report results under 50 random splits with $n_C = 35$ calibration and $n_{\mathrm{test}} = 5$ test blocks per split.

Table~\ref{tab:eobs_random} reports coverage and utility for global-uniform and local pooling. Both families pass the certified audit in nearly all splits (PASS $\ge 0.98$), confirming that the harm budget is not overly restrictive. Global pooling delivers substantial certified relief: a 27\% reduction in aggregate certified caps (AggCap 0.726) and nearly 50\% reduction for the highest-risk decile (Top10 0.506), demonstrating strong performance on \textbf{R2}. Coverage remains near-nominal (mean 0.919, p05 0.896), though closer to the boundary than identity, reflecting the cost of redistributing tail exposure. Local pooling is more conservative: it achieves moderate relief (AggCap 0.895, Top10 0.695) with higher empirical coverage (mean 0.946, p05 0.904), illustrating the safety-utility trade-off inherent in the Certified Allocation Problem. Broader pooling extracts more diversification but pushes coverage closer to nominal for some agents; local pooling sacrifices efficiency for empirical conservatism under spatial dependence.

To stress-test robustness to temporal drift, we repeat the analysis with time-ordered splits and vary the calibration length
$n_C \in \{10, 20, 30, 40, 50\}$ (Table~\ref{tab:eobs_time_nc}). Blocks are partitioned temporally (Appendix~\ref{sec:nonstationarity}). For global
pooling, increasing $n_C$ worsens out-of-window coverage (mean drops from 0.868 to 0.742; p05 from 0.745 to 0.564), consistent with a bias-variance trade-off under drift: larger calibration windows reduce quantile variance but incorporate data from earlier, less representative periods. Local pooling is more robust, maintaining higher coverage across all $n_C$.
Importantly, identity caps also degrade under time splits (Appendix~\ref{app:eobs-identity}), confirming that this is a genuine nonstationarity effect rather than a method failure: over 1950--2024, both trigger rates and conditional severity exhibit statistically significant upward trends, with aggregate losses increasing by roughly $50\%$ between the first and second halves of the record (Appendix~\ref{app:eobs-diagnostics}). 
Hence, conformal certificates may require periodic re-certification.

\begin{table}[t]
\centering
\setlength{\tabcolsep}{3pt}
\scriptsize
\begin{tabular}{lccccc}
\toprule
 & PASS & $\alpha_{\mathrm{op}}$ & Cov (p05) & AggCap & Top10 \\
\midrule
Global & 0.98 & 0.56\tiny$\pm$0.10 & 0.919 (0.896) & 0.726\tiny$\pm$0.047 & 0.506\tiny$\pm$0.087 \\
Local & 1.00 & 1.00\tiny$\pm$0.00 & 0.946 (0.904) & 0.895\tiny$\pm$0.012 & 0.695\tiny$\pm$0.022 \\
VO-DS & 0.98 & 0.57\tiny$\pm$0.10 & 0.920 (0.900) & 0.718\tiny$\pm$0.047 & 0.493\tiny$\pm$0.088 \\
Identity & -- & 0 & 0.972 (0.932) & -- & -- \\
\bottomrule
\end{tabular}
\caption{E-OBS precipitation, random splits ($n\!=\!1120$ peers,
$B\!=\!75$ years, $\delta\!=\!0.1$, $\varepsilon\!=\!0.2$, $n_C\!=\!35$,
50 splits).}
\label{tab:eobs_random}
\end{table}

\begin{table}[t]
\centering
\setlength{\tabcolsep}{3pt}
\scriptsize
\begin{tabular}{llccccc}
\toprule
 & $n_C$ & PASS & Mean & p05 & Min & Fr$<$.9 \\
\midrule
\multirow{5}{*}{\rotatebox{90}{\tiny Global}} & 10 & 0.91 & 0.868 & 0.745 & 0.600 & 0.836 \\
 & 20 & 0.91 & 0.800 & 0.709 & 0.509 & 0.796 \\
 & 30 & 1.00 & 0.756 & 0.600 & 0.491 & 0.874 \\
 & 40 & 1.00 & 0.762 & 0.600 & 0.400 & 0.820 \\
 & 50 & 1.00 & 0.742 & 0.564 & 0.345 & 0.854 \\
\midrule
\multirow{5}{*}{\rotatebox{90}{\tiny Local}} & 10 & 1.00 & 0.919 & 0.782 & 0.545 & 0.333 \\
 & 20 & 1.00 & 0.908 & 0.727 & 0.545 & 0.383 \\
 & 30 & 1.00 & 0.893 & 0.691 & 0.527 & 0.411 \\
 & 40 & 1.00 & 0.882 & 0.618 & 0.473 & 0.429 \\
 & 50 & 1.00 & 0.868 & 0.600 & 0.382 & 0.466 \\
\bottomrule
\end{tabular}
\caption{E-OBS time-ordered splits, varying $n_C$ ($n_{\mathrm{test}}\!=\!5$). Fr$<$.9: fraction of agents below nominal.}
\label{tab:eobs_time_nc}
\end{table}

\subsection{Energy Cooperative}
\label{sec:exp-energy}

To demonstrate generality beyond climate insurance, we apply the framework to an electricity consumption dataset from a Portuguese energy cooperative (CEL Loureiro) comprising $n = 153$ households observed over $B = 69$ weekly blocks (May 2022--September 2023) \citep{monteiro2024electricity}. We define each household's weekly loss as the excess consumption above a rolling seasonal baseline (9-week centered median), removing seasonal effects so that blocks are approximately exchangeable. The resulting losses are continuous, heavy-tailed, 
and moderately zero-inflated (zero fraction $0.58$). Unlike the E-OBS precipitation data, pairwise correlations across households are weak and unstructured (Appendix~\ref{app:energy-dependence}), reflecting largely idiosyncratic demand shocks. The observation window $B$ is too short for meaningful time-ordered splits in our split configuration, and we find no significant nonstationarity in the deseasonalised losses, so random splits are appropriate.
We report results under 50 splits with $n_C = 35$, $n_{\mathrm{test}} = 5$, and global-uniform pooling at two participation budgets.

Table~\ref{tab:energy_random} reports results for $\varepsilon = 0.05$ (tight budget) and $\varepsilon = 0.20$ (permissive). Coverage is near-nominal in both cases, consistent with the identity baseline (\textbf{R1}). The participation budget $\varepsilon$ directly controls the deployed mutualisation level: at $\varepsilon = 0.05$, the method selects $\alpha_\mathrm{op} = 0.31$ and delivers a 20\% aggregate cap reduction (AggCap $0.795$; Top10 ratio $0.719$); at $\varepsilon = 0.20$, it selects $\alpha_\mathrm{op} = 0.93$ and delivers a 51\% aggregate reduction with 82\% reduction in certified caps for the highest-demand decile (Top10 $0.180$) (\textbf{R2}). The PASS rate is $1.00$ at both levels, indicating that pooling benefits nearly all households simultaneously: the weak dependence structure means global averaging rarely harms any individual agent, so the harm constraint is slack (\textbf{R3}). This contrasts with the E-OBS setting, where spatially structured correlations cause pooling to harm some agents and the certification gate rejects a fraction of candidates. In all datasets, VO-DS produces a base rule close to global, reflecting the approximately symmetric dependence structure (Appendix~\ref{sec:VO-DS}). Additional sensitivity analyses over the participation budget $\varepsilon$ and target miscoverage level $\delta$ are reported in Appendix~\ref{app:energy-sensitivity}. The results exhibit the expected monotone efficiency–conservatism tradeoffs while maintaining near-nominal coverage.

\begin{table}[t]
\centering
\setlength{\tabcolsep}{3pt}
\scriptsize
\begin{tabular}{lccccc}
\toprule
 & PASS & $\alpha_{\mathrm{op}}$ & Cov (p05) & AggCap & Top10 \\
\midrule
Global ($\varepsilon\!=\!0.05$) & 1.00 & 0.31\tiny$\pm$0.08 & 0.914 (0.886) & 0.795\tiny$\pm$0.048 & 0.719\tiny$\pm$0.070 \\
Global ($\varepsilon\!=\!0.20$) & 1.00 & 0.93\tiny$\pm$0.10 & 0.914 (0.900) & 0.492\tiny$\pm$0.058 & 0.180\tiny$\pm$0.079 \\
VO-DS ($\varepsilon\!=\!0.05$) & 1.00 & 0.32\tiny$\pm$0.08 & 0.913 (0.882) & 0.787\tiny$\pm$0.048 & 0.708\tiny$\pm$0.070 \\
VO-DS ($\varepsilon\!=\!0.20$) & 1.00 & 0.94\tiny$\pm$0.09 & 0.915 (0.900) & 0.490\tiny$\pm$0.057 & 0.175\tiny$\pm$0.075 \\
Identity & -- & 0 & 0.914 (0.884) & -- & -- \\
\bottomrule
\end{tabular}
\caption{Energy cooperative, random splits ($n\!=\!153$ households,
$B\!=\!69$ weeks, $\delta\!=\!0.1$, $\eta\!=\!0$, $n_C\!=\!35$, 50 splits).
Participation budgets $\varepsilon$ illustrate efficiency-harm
trade-off.}
\label{tab:energy_random}
\end{table}

\section{Discussion and Conclusion}
\label{sec:discussion}

We introduced the Certified Allocation Problem, requiring joint policy selection, per-agent obligation caps, and participation verification from finite data. We proposed Conformal Risk Sharing as the first solution framework, demonstrating on synthetic and real data that it delivers substantial tail relief for high-risk peers, and harm control within explicit participation budgets.
The conformal certificate provides marginal, per-agent control of tail exceedance under block exchangeability. This yields finite-sample validity without parametric assumptions, but can be conservative when calibration blocks are few. Additionally, the guarantee is tied to block exchangeability; in practice this motivates periodic re-certification as new blocks arrive.

\paragraph{Limitations.}
The one-parameter policy class is interpretable and auditable but limits expressiveness. The conformal guarantee is marginal per agent and does not imply joint coverage across agents. Corollary~\ref{cor:system} provides a separate system-level certificate for user-chosen system-level functionals. Under strong dependence, participation constraints may be infeasible; in such cases the method transparently reports that material improvement is unattainable within the tested policy class at the requested safety level. Finally, nonlinear mechanisms (\emph{e.g.}, deductibles) may be more capital-efficient in some regimes but are not covered by the current linear model.

\paragraph{Future work.}
The Certified Allocation Problem admits solutions beyond conformal prediction. A Bayesian approach could regularise estimation in the small-sample regime through informative priors, but the resulting certificates are credible intervals whose coverage depends on correct prior specification rather than holding distribution-free as in the conformal case. A hierarchical Bayesian model across agents could exploit the spatial/cooperative structure (shared hyperparameters across agents) to produce tighter per-agent certificates.
Parametric approaches (\emph{e.g.}, extreme value theory, copula models) could estimate tail quantities under distributional assumptions, providing efficiency when models are well-specified but lacking distribution-free validity. 
Investigating these alternatives is a direction for future
work.

Several extensions of the framework merit investigation.
\emph{Richer policy classes:} parameterising the base rule $\bar A$ itself (\emph{e.g.}, via a learned kernel bandwidth or sparse graph weights) would move from a scalar to a low-dimensional search while preserving interpretability, and the conformal certificate remains valid for any policy fixed before calibration. Operationally, this may include layered insurance rules with deductibles, and cashback thresholds parameterised by a small number of tunable parameters.
\emph{Feature-conditional certificates:} replacing unconditional conformal caps with conformalised quantile regression \citep{romano2019conformalized} could yield tighter, covariate-adaptive certificates when side information (climate indices, exposure features) is available. 
\emph{Dynamic and multiperiod settings:} extending the framework to sequential certification, where the calibration window rolls forward and agents may enter or leave the pool, connects to online conformal prediction \citep{oliveira2024split} and multiperiod P2P insurance models \citep{abdikerimova2024multiperiod}.


\bibliography{ref}

\newpage

\onecolumn

\title{Conformal Risk Sharing: Certified Cost Allocation with Participation Guarantees \\(Supplementary Material)}
\maketitle

\appendix

\section{P2P Parametric Insurance: Background and Motivation}
\label{sec:p2p_parametric_appendix}

This appendix provides additional context on the peer-to-peer parametric insurance application that serves as one of the motivating examples throughout the paper.

\paragraph{Parametric insurance.}
A parametric insurance product triggers a fixed payout when an independently verifiable index crosses a predetermined threshold (\emph{e.g.}, rainfall exceeding a level, wind speed above a limit, a seismic intensity measure, or a commodity price index, all of which can be verified by independent institutions/providers), rather than indemnifying individually assessed losses~\citep{SwissRe}. This design reduces administrative overhead, eliminates the need for claims adjustment, and enables rapid, transparent settlement. However, sustainability still depends on effective risk pooling, which is challenging when trigger events are correlated across policyholders and when the underlying hazard distribution may shift over time.

\paragraph{Peer-to-peer risk sharing.}
Peer-to-peer (P2P) insurance decentralises risk pooling: participants collectively fund losses within a network rather than transferring risk to a corporate insurer in exchange for premiums \citep{feng2023decentralized,stoeckli2018exploring}. The model distributes financial responsibility across participants, reducing overhead and aligning incentives (members share surpluses via cashback and bear shortfalls collectively). This structure is particularly relevant in settings where traditional insurers have withdrawn due to high risk or where premiums are prohibitively high, leaving protection gaps that decentralised pools can partially fill.

\paragraph{What makes P2P mechanism design hard.}
Traditional insurance relies on centralised capital: reserves and reinsurance absorb tail outcomes, and the law of large numbers makes per-policy costs predictable at scale. P2P schemes are typically \emph{capital-light}, with limited central reserves and continued reliance on a heterogeneous membership. This creates a distinct design constraint: the mechanism must not only be balanced in expectation, but must also control extreme contribution outcomes. Three challenges are particularly salient:

\begin{itemize}[leftmargin=*,itemsep=2pt]
\item \emph{Heavy tails and dependence.} Trigger events can produce correlated payouts across many members simultaneously (\emph{e.g.}, a regional storm or a market-wide shock), concentrating burden and limiting the diversification benefit of pooling.
\item \emph{Voluntary participation.} Members who perceive their worst-case obligation as too high will leave, and selective exit by low-risk members can unravel the pool through adverse selection \citep{kocherlakota1996implications,ligon2005formation}.
\item \emph{Fairness and transparency.} If the mechanism systematically shifts tail burden onto a subset of members to subsidise others, those members exit first, undermining the pool both ethically and practically.
\end{itemize}

These challenges motivate the Certified Allocation Problem formalised in Sec.~\ref{sec:problem}: the mechanism designer needs per-agent tail certificates (\textbf{R1}) to assure participants of bounded exposure, aggregate efficiency (\textbf{R2}) to justify the existence of the pool, and explicit participation constraints (\textbf{R3}) to prevent the adverse selection spiral.

\paragraph{Connection to our framework.}
In our experiments (Sec.~\ref{sec:exp-eobs}), we instantiate this setting using gridded precipitation data with a parametric seasonal rainfall trigger as a concrete example. Each grid cell corresponds to a peer, each year is a block, and the trigger produces nonnegative, heavy-tailed, zero-inflated loss vectors. The allocation matrix $A^\star$ determines how realised payouts are redistributed, the conformal certificate $c_i(A^\star)$ provides each peer with a high-confidence obligation cap, and the participation constraints control the aggregate harm. The framework itself is agnostic to the specific trigger or hazard type.

\subsection{Additional Actuarial References}
\label{app:actuarial-context}

This appendix provides additional context for readers from actuarial science and insurance, complementing the shorter discussion in Sec.~\ref{sec:related}. The main distinction is that much of the actuarial risk-sharing literature studies population-level allocation rules under known or modelled loss distributions, whereas our focus is finite-sample, distribution-free certification of a selected allocation policy.

\paragraph{Distribution-dependent risk-sharing rules.}
A central class of actuarial mechanisms allocates losses as functions of the aggregate pool loss~\citep{denuit2012convex,denuit2022risk,dhaene2025axiomatic}. These rules have attractive axiomatic and Pareto-efficiency properties, but require knowledge of the joint distribution, conditional expectations, or conditional quantiles. In contrast, the Certified Allocation Problem assumes only finitely many observed cost vectors and asks for data-derived caps with finite-sample validity.

\paragraph{Participation, cashback, and incentives.}
P2P and mutual insurance designs often introduce cashback, side payments, or Shapley-value allocations to maintain participation incentives \citep{clemente2023optimal,clemente2024risk}. These mechanisms usually express individual rationality in expected-value or surplus-sharing terms. Our participation constraints play a related role, but are enforced in certified tail-cap units: the mechanism is deployed only if the certified harm relative to the identity baseline is below the chosen budget. Strategic behaviour and moral hazard are not modelled here; incorporating them would require an additional incentive-compatibility layer.

\paragraph{Dynamic and multiperiod settings.}
Multiperiod P2P insurance models study reserves, solvency, reinsurance layers, and inter-temporal utility under specified stochastic models \citep{abdikerimova2024multiperiod}. Extending our framework to sequential deployment would require calibration sets that adapt to a moving distribution. This connects naturally to conformal inference under dependence and nonstationarity, and to the recertification questions raised by the E-OBS time-split experiments.

\section{Methodological details}

\subsection{Nonstationarity and Time-Split Protocols}
\label{sec:nonstationarity}

The conformal guarantee requires that deployment blocks are exchangeable with calibration blocks. When the loss distribution drifts over time (as in climate-driven applications), this is best interpreted as a \emph{local stationarity} condition: coverage holds for blocks from the same regime as $\mathcal{C}$. To operationalise this, we partition blocks in temporal order such that the calibration window is as close as possible to the deployment period. Under nonstationarity, certificates should be viewed as a rolling operational contract rather than a permanent guarantee: one re-certifies periodically (\emph{e.g.}, annually) as new blocks arrive. In our experiments we report both random splits (theorem-aligned) and time-ordered splits (robustness diagnostic).

\subsection{Variance-Optimal Doubly-Stochastic baseline (VO-DS).}\label{sec:VO-DS}

To benchmark against classical second-moment risk sharing (not tailored to tail risk), we include a \emph{variance-optimal doubly-stochastic} baseline. Under our convention in~\eqref{eq:alloc}, post-sharing obligations are $x_b(A)=\tilde x_b A$, so if $\Sigma=\mathrm{Cov}(\tilde x_b)$ then $\mathrm{Cov}(\tilde x_b A)=A^\top \Sigma A$. Using the training split, we estimate a shrunken covariance
\[
\widehat\Sigma_\lambda
=(1-\lambda)\widehat\Sigma+\lambda\,\mathrm{diag}(\widehat\Sigma)+\rho_{\text{ridge}} I,
\]
and define the VO-DS reference rule as a minimiser of the quadratic proxy
\begin{equation}
\bar A_{\mathrm{VO}} \in \argmin_{A\in\mathcal{A}_{\mathrm{DS}}(M)}
\;\mathrm{tr} \big(A^\top \widehat\Sigma_\lambda A\big),
\label{eq:vods}
\end{equation}
over the set of admissible doubly-stochastic allocation matrices
\[
\mathcal{A}_{\mathrm{DS}}(M)
=
\left\{
A\in\mathbb{R}^{n\times n}:
A\ge 0,\;
A\mathbf 1=\mathbf 1,\;
A^\top\mathbf 1=\mathbf 1,\;
A_{ji}=0 \text{ whenever } M_{ji}=0
\right\}.
\]
Here $M\in\{0,1\}^{n\times n}$ encodes admissible sharing links (for the unconstrained VO-DS baseline we take $M\equiv \mathbf 1\mathbf 1^\top$). In our implementation, \eqref{eq:vods} is solved approximately on the training split using a projected first-order method, and the resulting $\bar A_{\mathrm{VO}}$ is then treated as a fixed base rule in the train--select--certify pipeline. 

In all three datasets in this paper, VO-DS produces results that closely match global uniform pooling. This behaviour is consistent with prior linear risk-sharing results for doubly stochastic mixing, where equal sharing emerges as an extremal case on complete graphs~\citep{charpentier2021collaborative}. This suggests that, within the class of doubly-stochastic sharing rules considered here, the overall intensity of redistribution may matter more than its precise structure. We leave a systematic study of richer sharing-rule classes for future work.

\section{Experimental details}
\subsection{Synthetic Data Generating Process}
\label{app:synth-dgp}

We generate $B$ exchangeable blocks (years) $b=1,\dots,B$ with $n$ peers arranged on a grid. Each year is an event year with probability
$p_{\mathrm{event}}$; on a non-event year all losses are zero. On an event year $b$, peer $i$ incurs a payout
\[
  X_{b,i} = \mathbf{1}\{\mathrm{hit}_{b,i}=1\} \cdot S_b \cdot E_i,
\]
where $E_i > 0$ is a peer-specific exposure and $S_b > 0$ is a year level severity. Conditional on being an event year, hits are generated via a logistic model
\[
  P(\mathrm{hit}_{b,i}=1 \mid Z_b, S^{\mathrm{sp}}_{b,i})
  = \sigma \big(\beta_0 + \kappa\, r_i + \lambda\, Z_b
    + \rho\, S^{\mathrm{sp}}_{b,i}\big),
\]
with $\sigma(\cdot)$ the sigmoid, $r_i$ an i.i.d.\ peer risk score
(heterogeneity), $Z_b$ a year-level common factor (common shocks), and
$S^{\mathrm{sp}}_{b,i}$ an optional spatially smoothed random field (spatial dependence). The intercept $\beta_0$ is set so that $\sigma(\beta_0) \approx p_{\mathrm{hit}|\mathrm{event}}$ for a typical peer. Exposures $E_i$ are i.i.d.\ lognormal normalised to mean~1. Year severities $S_b$ are Pareto distributed.

\subsection{Evaluation Metrics}
\label{app:synth-metrics}

Let $\mathcal{S}$ denote the set of splits and
$\mathcal{B}_{\mathrm{test}}(s)$ the test blocks for split $s$. For the
deployed policy/certificate pair $(A_{\mathrm{op}}, c_{\mathrm{op}})$:

\emph{Per-agent marginal coverage:}
\[
  \widehat{\mathrm{cov}}_i
  = \frac{1}{\sum_{s} |\mathcal{B}_{\mathrm{test}}(s)|}
    \sum_{s \in \mathcal{S}} \sum_{b \in \mathcal{B}_{\mathrm{test}}(s)}
    \mathbf{1}\{X_{b,i}(A_{\mathrm{op}}) \le (c_{\mathrm{op}})_i\}.
\]
We report the mean and 5th percentile (p05) of
$\{\widehat{\mathrm{cov}}_i\}_{i=1}^n$ across agents.

\emph{Aggregate certified-cap ratio:}
$\mathrm{AggCapRatio}
  = \langle w, c_{\mathrm{op}} \rangle / \langle w, c_0 \rangle$,
where $c_0$ are the identity baseline caps. A value of 1 means pooling offers no improvement over the baseline, values below 1 indicate that pooling reduces the aggregate high-confidence upper bounds on obligations (lower is better)

\emph{Top-decile cap ratio:} the same ratio restricted to the top 10\% of
agents ranked by baseline cap $c_{0,i}$ (lower values indicate larger certified
relief for the highest-risk agents).

\emph{Fraction below nominal (Fr$<.9$):} the proportion of agents whose empirical marginal coverage falls below the nominal $1-\delta$ level (lower is better, zero means all agents are at or above nominal).

\emph{PASS rate:} fraction of splits where the candidate passes the certified harm audit and is deployed (higher indicates more splits where pooling is acceptable).

\subsection{E-OBS: Calibration Length Sensitivity Under Random Splits}
\label{app:eobs-nc-random}

\begin{table}[t]
\centering
\setlength{\tabcolsep}{3pt}
\small
\begin{tabular}{llccccc}
\toprule
 & $n_C$ & PASS & Mean & p05 & Min & Fr$<$.9 \\
\midrule
\multirow{5}{*}{\rotatebox{90}{\tiny Global}} & 10 & 0.94 & 0.883 & 0.856 & 0.824 & 0.848 \\
 & 20 & 0.98 & 0.890 & 0.864 & 0.848 & 0.808 \\
 & 30 & 0.88 & 0.895 & 0.868 & 0.840 & 0.797 \\
 & 40 & 1.00 & 0.884 & 0.860 & 0.832 & 0.854 \\
 & 50 & 0.96 & 0.885 & 0.860 & 0.840 & 0.866 \\
\midrule
\multirow{5}{*}{\rotatebox{90}{\tiny Local}} & 10 & 1.00 & 0.929 & 0.868 & 0.832 & 0.322 \\
 & 20 & 1.00 & 0.929 & 0.872 & 0.852 & 0.344 \\
 & 30 & 1.00 & 0.926 & 0.872 & 0.840 & 0.375 \\
 & 40 & 1.00 & 0.926 & 0.868 & 0.844 & 0.404 \\
 & 50 & 1.00 & 0.925 & 0.868 & 0.844 & 0.408 \\
\bottomrule
\end{tabular}
\caption{E-OBS random splits, varying $n_C$ ($n_{\mathrm{test}}\!=\!5$, 50 splits). Fr$<$.9: fraction of agents below nominal.}
\label{tab:eobs_rand_nc}
\end{table}

Table~\ref{tab:eobs_rand_nc} complements the time-split sensitivity analysis in Table~\ref{tab:eobs_time_nc} by varying $n_C$ under random splits. Random splits remove the explicit temporal extrapolation in the time-ordered protocol, but do not eliminate the nonstationarity and dependence present in the E-OBS record.
They should therefore be interpreted as a diagnostic for finite-sample and post-pooling effects rather than as evidence that the climate blocks are truly exchangeable.

For local pooling, coverage is stable across all $n_C$ (mean $\approx 0.93$, p05 $\approx 0.87$) and the identity baseline achieves mean $\approx 0.96$ with p05 at or above $0.90$ throughout, confirming that the conformal procedure itself is valid and conservative.

Global pooling, however, exhibits systematic marginal undercoverage (mean $0.88$--$0.90$, p05 $0.86$, Fr$< 0.9$ around $80$--$87\%$) even under random splits. This does not reflect a failure of the conformal certificate itself: the identity caps are valid, and the certificate for the \emph{fixed deployed} policy is correct by Theorem~\ref{thm:main}. Rather, global pooling with $\alpha \approx 0.5$ applies an aggressive transformation that concentrates post-pooling obligations (each agent's obligation becomes roughly half their own loss plus half the group mean), thinning the effective tail from which conformal caps are estimated. With only $B = 75$ total blocks and correspondingly small calibration sets, the order-statistic caps for this transformed distribution sit closer to the true quantile, leaving less margin and resulting in empirical coverage slightly below nominal for a substantial fraction of agents.

Local pooling avoids this issue because it applies a milder transformation (neighbourhood averaging preserves more of the original per-agent distribution shape), and the resulting caps retain a conservative margin even at small $n_C$. This highlights a practical trade-off: more aggressive redistribution delivers greater certified tail relief (Table~\ref{tab:eobs_random}) but requires larger calibration sets for conformal caps to remain conservative. Feature-conditional approaches such as conformalised quantile regression \citep{romano2019conformalized} could help tighten caps for the transformed distribution, reducing this sensitivity.

\subsection{E-OBS: Identity Baseline Under Time Splits}
\label{app:eobs-identity}

Table~\ref{tab:eobs_identity_time} shows that the identity baseline also degrades under time-ordered splits, even though it involves no pooling. As $n_C$ increases from 10 to 50, mean coverage drops from 0.957 to 0.931 and the lower tail worsens substantially (p05: $0.836 \to 0.727$; min: $0.727 \to 0.491$), while the fraction of agents below nominal rises from $0.129$ to $0.252$. This confirms that the coverage deterioration observed in Table~\ref{tab:eobs_time_nc} is a genuine nonstationarity effect rather than an artifact of the learned policy or pooling choice.

Table~\ref{tab:eobs_identity_rand} reports the same analysis under random splits. Identity coverage is stable across all $n_C$ (mean $0.960$--$0.962$, p05 $0.900$--$0.904$), with fewer than $5.1\%$ of agents below nominal in all cases. This confirms that the conformal procedure is valid under exchangeability, and that the undercoverage observed for global pooling under random splits (Appendix~\ref{app:eobs-nc-random}) is attributable to the aggressive post-pooling transformation rather than a failure of the calibration procedure.

\begin{table}[t]
\centering
\setlength{\tabcolsep}{3pt}
\small
\begin{tabular}{ccccc}
\toprule
$n_C$ & Mean & p05 & Min & Fr$<$.9 \\
\midrule
10 & 0.957 & 0.836 & 0.727 & 0.129 \\
20 & 0.952 & 0.818 & 0.618 & 0.182 \\
30 & 0.945 & 0.782 & 0.527 & 0.203 \\
40 & 0.939 & 0.764 & 0.509 & 0.222 \\
50 & 0.931 & 0.727 & 0.491 & 0.252 \\
\bottomrule
\end{tabular}
\caption{Identity baseline under time-ordered splits on E-OBS. Coverage degrades with $n_C$ due to nonstationarity, confirming that drift affects all methods equally.}
\label{tab:eobs_identity_time}
\end{table}

\begin{table}[t]
\centering
\setlength{\tabcolsep}{3pt}
\small
\begin{tabular}{ccccc}
\toprule
$n_C$ & Mean & p05 & Min & Fr$<$.9 \\
\midrule
10 & 0.962 & 0.904 & 0.856 & 0.047 \\
20 & 0.962 & 0.904 & 0.876 & 0.043 \\
30 & 0.961 & 0.904 & 0.876 & 0.049 \\
40 & 0.961 & 0.904 & 0.856 & 0.047 \\
50 & 0.960 & 0.900 & 0.860 & 0.051 \\
\bottomrule
\end{tabular}
\caption{Identity baseline under random splits on E-OBS. Empirical coverage remains above nominal across all $n_C$.}
\label{tab:eobs_identity_rand}
\end{table}

\subsection{E-OBS: Nonstationarity Diagnostics}
\label{app:eobs-diagnostics}

The E-OBS precipitation record exhibits clear nonstationarity over 1950--2024. Aggregate trigger count trends upward at $+4.76$/year, with mean losses increasing by $53\%$ between the first and second halves of the record (from $378$ to $576$). Both trigger rate ($+37\%$, $p < 10^{-4}$) and conditional severity ($+11\%$, $p = 0.013$) contribute. Aggregate lag-1 autocorrelation is moderate ($r = 0.23$) but per-peer autocorrelation is weak (median $0.039$), indicating that the aggregate persistence is largely driven by trend rather than short-range temporal dependence. This confirms that nonstationarity is the dominant source of exchangeability violation in time-ordered splits.

\subsection{E-OBS: Sensitivity to Calibration Sample Size}
\label{app:eobs-sensitivity}

The observed undercoverage in some of the E-OBS experiments is partly attributable to the limited number of available calibration blocks ($B=75$). To better understand the undercoverage observed in the E-OBS stress test, we performed an additional diagnostic in which the training, validation, and test blocks were held fixed, and $\alpha^\star$ was selected once. We then restricted calibration to a temporally localised pool of years (using calibration blocks from the same half of the record as the test period) and varied the number $m$ of calibration blocks used for certification; see Table~\ref{tab:eobs_calib_local}. Increasing $m$ from 5 to 17 led to a monotonic improvement in empirical coverage (mean coverage increasing from 0.85 to 0.97) and a substantial reduction in the fraction of agents below nominal coverage (from 0.77 to 0.14). This suggests that part of the observed undercoverage is driven by limited availability of calibration data that are representative of the test period. However, even with the largest feasible localised calibration sets, some agents remained below nominal coverage, indicating that calibration size alone does not fully explain the effect.

\begin{table}[h]
\centering
\small
\begin{tabular}{cccc}
\toprule
$m$ & Cov Mean & Cov p05 & Fr$<$.9 \\
\midrule
5  & 0.851 & 0.666 & 0.767 \\
8  & 0.901 & 0.740 & 0.243 \\
10 & 0.915 & 0.773 & 0.192 \\
12 & 0.934 & 0.787 & 0.160 \\
14 & 0.946 & 0.800 & 0.151 \\
17 & 0.968 & 0.800 & 0.142 \\
\bottomrule
\end{tabular}
\caption{E-OBS calibration-size diagnostic under global pooling. Training, validation, and test blocks are fixed, $\alpha^\star$ is selected once, and certification is repeated using $m$ calibration blocks from a temporally localised pool.}
\label{tab:eobs_calib_local}
\end{table}

\subsection{Energy Cooperative: Data and Preprocessing}
\label{app:energy-data}

The CEL Loureiro dataset \citep{monteiro2024electricity} provides 15-minute smart meter readings for 172 buildings in a Portuguese energy cooperative (Loureiro, Portugal) from May 2022 to September 2023. We aggregate to weekly total consumption per building, discard buildings with $>50\%$ missing weeks (retaining $n = 153$), and drop partial weeks at the start and end of the record, yielding $B = 69$ full weekly blocks.

\paragraph{Deseasonalisation.}
Electricity consumption exhibits strong seasonality (most likely due to winter heating demand). To produce approximately exchangeable blocks, we define each household's weekly loss as the excess above a rolling seasonal baseline: for each building, we compute a 9-week centered rolling median and subtract it, clipping at zero: $\tilde x_{b,i} = \max(0,\, x_{b,i}^{\mathrm{raw}} - \mathrm{median}_9(x_{\cdot,i}))$. This removes the seasonal level while preserving genuine demand shocks. The resulting losses have zero fraction $0.58$, mean $9.2$, median $3.7$, and max $543$.

\subsection{Energy Cooperative and E-OBS: Dependence Structure}
\label{app:energy-dependence}

Table~\ref{tab:dependence_comparison} compares the dependence structure of the energy cooperative and E-OBS datasets. The energy cooperative exhibits weak, unstructured pairwise correlations (mean $r = 0.064$, no spatial block structure), while E-OBS shows spatially structured correlations (visible block diagonal in the correlation matrix) with a heavier right tail (95th percentile
$r = 0.45$ vs.\ $0.35$). This difference explains why global pooling benefits nearly all households in the energy cooperative (idiosyncratic shocks diversify effectively) but creates harmed agents in the E-OBS setting (correlated shocks limit diversification and shift burden onto low-trigger-rate cells).

\begin{table}[h]
\centering
\small
\begin{tabular}{lcc}
\toprule
& Energy coop. & E-OBS \\
\midrule
$n$ (agents) & 153 & 1120 \\
$B$ (blocks) & 69 weeks & 75 years \\
Mean pairwise $r$ & 0.064 & 0.080 \\
Median pairwise $r$ & 0.034 & 0.041 \\
95th percentile pairwise $r$ & 0.350 & 0.446 \\
Spatial structure & None & Block diagonal \\
\bottomrule
\end{tabular}
\caption{Dependence structure comparison. The energy cooperative has weaker,
unstructured correlations, explaining why pooling is nearly Pareto-improving
while E-OBS pooling creates harmed agents.}
\label{tab:dependence_comparison}
\end{table}

\subsection{Energy Cooperative: Sensitivity to Participation Budget and Miscoverage Level}
\label{app:energy-sensitivity}

We investigate the sensitivity of the Energy Cooperative experiment to the participation budget $\varepsilon$ and the target miscoverage level $\delta$. Throughout, we fix $\eta=0$ and report results averaged over 50 random splits.

\paragraph{Participation budget $\varepsilon$.}
The participation budget controls the maximum allowable increase in expected cost for any participant. As $\varepsilon$ increases, the feasible set expands, allowing more aggressive redistribution. Table~\ref{tab:eps_sweep} shows that the selected pooling intensity $\alpha_{\mathrm{op}}$ increases monotonically from $0.05$ to $0.93$ as $\varepsilon$ increases from $0.01$ to $0.20$. At the same time, certified-cap efficiency improves and the empirical coverage remains essentially unchanged across the sweep.
Overall, increasing $\varepsilon$ enables stronger redistribution and larger certified-cap reductions, while increasing $\delta$ relaxes the coverage requirement and yields less conservative allocations.

\begin{table}[h]
\centering
\small
\begin{tabular}{lcccccc}
\toprule
$\varepsilon$ & $\alpha_{\mathrm{op}}$ & Cov Mean & Cov p05 & PASS & AggCap & Top10 \\
\midrule
0.01 & 0.05 & 0.914 & 0.884 & 1.00 & 0.965 & 0.953 \\
0.05 & 0.31 & 0.914 & 0.886 & 1.00 & 0.795 & 0.719 \\
0.10 & 0.57 & 0.913 & 0.884 & 1.00 & 0.641 & 0.484 \\
0.20 & 0.93 & 0.914 & 0.900 & 1.00 & 0.492 & 0.180 \\
\bottomrule
\end{tabular}
\caption{Sensitivity to the participation budget $\varepsilon$ in the Energy Cooperative experiment.}
\label{tab:eps_sweep}
\end{table}

\paragraph{Miscoverage level $\delta$.}
We also vary the target miscoverage level $\delta$ as shown in Table~\ref{tab:delta_sweep}. Larger values of $\delta$ permit less conservative certification and therefore smaller pooling intensities. The selected $\alpha_{\mathrm{op}}$ decreases monotonically from $1.00$ at $\delta=0.05$ to $0.36$ at $\delta=0.20$. Empirical coverage tracks the nominal target coverage level $1-\delta$ throughout the sweep.

\begin{table}[h]
\centering
\small
\begin{tabular}{lccccccc}
\toprule
$\delta$ & Nominal & $\alpha_{\mathrm{op}}$ & Cov Mean & Cov p05 & PASS & AggCap & Top10 \\
\midrule
0.05 & 0.95 & 1.00 & 0.976 & 0.976 & 1.00 & 0.281 & 0.070 \\
0.10 & 0.90 & 0.93 & 0.914 & 0.900 & 1.00 & 0.492 & 0.180 \\
0.15 & 0.85 & 0.70 & 0.858 & 0.828 & 1.00 & 0.692 & 0.398 \\
0.20 & 0.80 & 0.36 & 0.804 & 0.764 & 0.98 & 0.904 & 0.700 \\
\bottomrule
\end{tabular}
\caption{Sensitivity to the target miscoverage level $\delta$ in the Energy Cooperative experiment.}
\label{tab:delta_sweep}
\end{table}

\end{document}